\documentclass{article}
\usepackage{graphicx} 

\usepackage{amsmath} 
\usepackage{amssymb}  
\usepackage{amsthm}  
\usepackage{bm}
\usepackage{enumitem}

\usepackage{lipsum}
\usepackage{color}

\newtheorem{defn}{Definition}
\newtheorem{cor}{Corollary}
\newtheorem{lem}{Lemma}
\newtheorem{thm}{Theorem}

\newtheorem{example}{Example}

\numberwithin{equation}{section}

\newcommand{\Mod}[1]{\ (\mathrm{mod}\ #1)}

\usepackage{amssymb}
\usepackage{hyperref}

\usepackage{graphicx}
\usepackage{caption}
\usepackage{subcaption}
\usepackage{comment}
\usepackage{algorithm}
\usepackage{algpseudocode}
\usepackage{booktabs}
\usepackage{multirow}
\usepackage{multicol}
\usepackage{xcolor}
\graphicspath{{./figures/}}

\title{Acceleration of Grokking in Learning Arithmetic Operations via Kolmogorov-Arnold Representation}

\author
{
Yeachan Park\thanks{Korea Institute for Advanced Study, 85 Hoegi-ro, Dongdaemun-gu,Seoul,02455, Seoul, South Korea}
 \and
 Minseok Kim \thanks{Deparment of Applied Artificial Intelligence, Seoul National University of Science and Technology, 232 Gongneung-ro, Nowon-gu, Seoul, 01811, South Korea}
 \and
 Yeoneung Kim\footnotemark[2]
}

\date{May 2024}

\begin{document}

\maketitle

\begin{abstract}
We propose novel methodologies aimed at accelerating the grokking phenomenon, which refers to the rapid increment of test accuracy after a long period of overfitting as reported in~\cite{power2022grokking}. Focusing on the grokking phenomenon that arises in learning arithmetic binary operations via the transformer model, we begin with a discussion on data augmentation in the case of commutative binary operations. To further accelerate, we elucidate arithmetic operations through the lens of the Kolmogorov-Arnold (KA) representation theorem, revealing its correspondence to the transformer architecture: embedding, decoder block, and classifier. Observing the shared structure between KA representations associated with binary operations, we suggest various transfer learning mechanisms that expedite grokking. This interpretation is substantiated through a series of rigorous experiments. In addition, our approach is successful in learning two nonstandard arithmetic tasks: composition of operations and a system of equations. Furthermore, we reveal that the model is capable of learning arithmetic operations using a limited number of tokens under embedding transfer, which is supported by a set of experiments as well.

\end{abstract}


\section{Introduction}
Deep artificial neural networks (NNs) have received lots of attention thanks to their capability of representing a wide class of functions and their success in various tasks ranging from classification, image, or text generation to reinforcement learning. The training of deep NNs aims at minimizing training error while maintaining high generalization performance which is computed with unseen data during the training procedure. However, it is often observed that validation accuracy does not reach the desired level even if the training error is small enough, particularly when a limited number of the dataset is available. We regard this status as overfitting. Recently, it has been pointed out that long after overfitting, validation accuracy sometimes quickly increases toward perfect generalization, and such a phenomenon is called ‘grokking’ as proposed in~\cite{power2022grokking}. 

Learning arithmetic operations using deep neural networks is challenging due to its poor generalization performance~\cite{trask2018neural} and the necessity of large dataset~\cite{hoshen2016visual} as well as high accuracy. The grokking phenomenon was initially identified in a study on learning arithmetic operations~\cite{power2022grokking}, wherein the model significantly enhances its generalization capabilities following prolonged training on small, algorithmically generated datasets. This phenomenon is also observed beyond algorithmic datasets, encompassing images, language, molecular data, and sparse parities problem as noted in~\cite{liu2022omnigrok,barak2022hidden,chughtai2023toy,charton2024learning}. In particular,~\cite{barak2022hidden} provides a rigorous justification for the grokking in learning sparse parities problem.

In this work, we propose a rigorous approach for accelerating grokking in learning arithmetic operations. Leveraging the advantages of algebraic structures, we identify the correspondence between the Kolmogorov-Arnold (KA) representation and the transformer architecture, which is a central idea behind the framework of transfer learning. Furthermore, we test the validity of our approach for some new extended tasks, including learning the composition of arithmetic operations and solving systems of equations with unknowns

\subsection{Related works}
In recent years, several attempts have been made to address the mechanism behind grokking. The authors of~\cite{liu2022omnigrok} propose that the grokking phenomenon is simply induced by the choice of large initialization and weight norms, based on the observation of the Goldilocks zone~\cite{fort2019goldilocks}. Another viewpoint for understanding grokking is suggested by~\cite{thilak2022slingshot}, where it is argued that the slingshot mechanism leads to grokking. The authors of~\cite{liu2022towards} explain how structured representations emerge and contribute to grokking by classifying the learning process into four phases: comprehension, grokking, memorization, and confusion. On the other hand, ~\cite{nanda2022progress,zhong2024clock} focuses on learning modular addition by leveraging discrete Fourier transforms and trigonometric identities, which transform modular addition to rotations around a circle. Slightly later,~\cite{lyu2023dichotomy} explores the necessary conditions for grokking by proving that such a phenomenon occurs around the Karush–Kuhn–Tucker point in sparse linear classification and matrix completion problems. Recently, the authors of~\cite{tan2023understanding} propose an approach for degrokking, speeding up the generalization process, by perturbing the loss function. The work most related to ours is~\cite{furuta2024interpreting}, where authors implement the idea of weight transfer for learning different binary operations from other pretrained models. However, its theoretical foundation is questionable. In our paper, we deconstruct arithmetic operations and their compositions via the KA representation theorem to unveil the relationship with the transformer architecture, allowing us to leverage weight transfer in learning arithmetic operations.

\subsection{Our contribution}
The contribution of the paper can be summarized as follows: 
\begin{itemize}
    \item It is empirically verified that the commutative augmentation technique introduced accelerates grokking.
    
    \item We revisit the KA representation theorem and interpret the model as a combination of three modules: the embedding, decoder block, and classifier module. We present a refined version of the KA theorem for abelian and anti-abelian operations, which allows us to validate the implementation of transfer learning techniques. The acceleration of grokking is achieved, as verified in diverse experiments.

    \item Two novel arithmetic tasks that include the composition of arithmetic operations and a system of equations with unknowns, are proposed. We demonstrate that transfer learning contributes to the efficient learning of these tasks, outperforming vanilla approaches.
    
    \item Finally, we demonstrate that training with a limited number of tokens is also possible via transfer learning, which supports the claim that the model actually understands arithmetic rules.
\end{itemize}

\section{Preliminary}
\subsection{Binary operation and group structure}
A binary operation on a set $G$ is a mapping of the elements in $G \times G$ to $G$ and is often denoted by $\circ$, that is, $a \circ b \in G$ for any $a,b \in G$. When $a \circ b = b \circ a$ for any $a,b\in G$, we call the operation is commutative. 

If a pair $<G,\circ>$ is further equipped with the following three rules, we call $G$ a group.
\begin{itemize}
\item Associativity property: for any $a,b,c\in G$, $(a\circ b) \circ c = a\circ (b\circ c)$;
\item Identity element: there exists $e\in G$ such that $e\circ a = a \circ e = a$ for all $a\in G$;
\item Inverse element: for any $a \in G$, there exists $b \in G$ such that $a \circ b = b \circ a  = e$.
\end{itemize}
We also denote the number of elements in $G$ \textit{order}.

To illustrate, let $G=\mathbb{Z}$ and the binary operation be defined as $a \circ b  :=a^2 + ab$. Since $4 =1 \circ (2 \circ 1) \neq (1 \circ 2) \circ 1 = 10$, the associativity rule is not satisfied, hence, $<G,\circ>$ is not a group. On the other hand, 
let us consider the set of integers  modulo $p$ for a prime number $p$, that is, $\mathbb{Z}_p:=\{ 0, 1 , \dots , p-1 \}$. For $x,y \in \mathbb{Z}_p$, the addition and multiplication under modular arithmetic is given as
\begin{align*}
    &x + y  = z, \quad \text{if} \quad  x+y =  z \Mod{p}, \\
    &x \times y  = z, \quad \text{if} \quad  x \times y = z \Mod{p},
\end{align*}
and they form groups. In particular, if $a \circ b = b \circ a$ for any $a,b\in G$ where $G$ is a group, $G$ is called an abelian group. Therefore, $\mathbb{Z}_p$ forms an abelian group with respect to both addition and multiplication.

\subsection{More on abelian groups}
Let us discuss well-established properties of abelian groups.
\begin{defn}[cyclic group]
A group $<G,\circ>$ is cyclic if it is generated by a single element $g \in G$, that is, $G=\{ g^k : k: \mathbb{Z} \}$. We refer $g$ as the generator of $G$.
\end{defn}

It is known that any finite abelian group is expressed as a product of finite cyclic groups. 
\begin{thm}[fundamental theorem of finite abelian group]\cite[Theorem 11.1]{gallian2021contemporary}
Let $G$ be a finite abelian group with the operation $\circ$. Then 
\[
G \cong C_{q_1} \times \dots \times C_{q_m},
\]
where $C_{q_j}$ is a cyclic group of order $q_j$ and $m$ denotes the number of generators of $G$.
\end{thm}
Of our interests are learning (i) general commutative arithmetic operations (ii) arithmetic operations leveraging abelian group structures, (iii) some non-commutative arithmetic operations such as subtraction and division, (iv) the composition of arithmetic operations, (v) a system of equations with unknowns. 

\section{Kolmogorov-Arnold (KA) representation}
In this section, we recall the KA representation theorem stating that every multivariate continuous function can be expressed as a superposition of two functions. Let us provide a formal version of the KA representation theorem and discuss its extension for functions defined on abelian groups. A core idea is to regard the composition of arithmetic operations as a function of $n\in\mathbb{N}$ variables, that is,
\[
f_n(x_1,...,x_n):=x_1\circ x_2 \circ...\circ x_n.
\]
\subsection{Representation of multivariate functions}

\begin{thm}[Kolmogorov-Arnold representation \cite{braun2009constructive}]
\label{thm:KA}
    For any arbitrary continuous function $f_n: \mathbb{R}^n \to \mathbb{R}$, there exist 
    \[
    \psi_{f_n}:\mathbb{R}^{2n+1}\to\mathbb{R} \quad \text{and} \quad \phi_{f_n}:\mathbb{R}\to\mathbb{R}^{2n+1},
    \]
    such that 
    \begin{align*}
        f_n(x_1,\dots,x_n) = \psi_{f_n}( \sum_{i=1}^{n} \lambda_i \phi_{f_n}(x_i)).
    \end{align*}
In particular, for any binary operation $\circ$, there exist 
\[
\phi:\mathbb{R} \to \mathbb{R}^5 \quad \text{and} \quad \psi:\mathbb{R}^5 \to \mathbb{R},
\]
such that
\begin{align*}
    x_1 \circ x_2 = \psi( \sum_{i=1}^2 \lambda_i \phi(x_i)),
\end{align*}
which indicates that the embedding dimension is $5$.
\end{thm}
Throughout the paper, we call $\phi$ and $\psi$ an embedding and an outer function respectively. The KA representation theorem states that any continuous function can be represented by the composition of two functions $\phi,\psi$ with the embedding dimension of $2n+1$. 

\begin{example}
Note that the KA representation may not be unique. For example, consider $f(x_1,x_2)= x_1 x_2$. Consider the following two representations:
    \begin{align*}
        &\phi_1(x)=[x,x^2], \;  \psi_1(u,v)=\frac{1}{2}(u^2-v),\; \lambda_i=1,\; i=1,2, \\
        &\phi_2(x)=\log(x), \;  \psi_2(u)=\exp(u),\; \lambda_1=1.
    \end{align*}
Hence, we have that
\begin{align*}
    \psi_1( \phi_1(x_1) + \phi_1(x_2)) 
    &= \psi_1( [x_1+x_2, x_1^2+x_2^2]) \\
    &= \frac{1}{2} \big( (x_1+x_2)^2-(x_1^2+x_2^2) \big) \\
    &= x_1x_2
\end{align*}
and
\begin{align*}
    \psi_2( \phi_2(x_1) + \phi_2(x_2)) &= \psi_2( \log(x) + \log(x_2)) \\
    &= \exp(\log(x_1x_2))\\
    &=x_1x_2,
\end{align*}
which verifies that the KA representation of $f(x_1,x_2)=x_1x_2$ is not unique.
\end{example}

If we further assume that $f$ is permutation-invariant, then we have a simpler representation with the reduced embedding dimension of n+1 as introduced in~\cite{zaheer2017deep}.  
\begin{thm}[permutation-invariant representation \cite{zaheer2017deep}]\label{thm:emb_n}
    For any permutation-invariant arbitrary continuous function $f_n: \mathbb{R}^n \to \mathbb{R}$, there exist 
    \[
    \psi_{f_n}:\mathbb{R}^{n+1}\to\mathbb{R} \quad  \text{and} \quad  \phi_{f_n}:\mathbb{R}\to\mathbb{R}^{n+1},     
    \]
such that 
    \begin{align*}
        f_n(x_1,\dots,x_n) = \psi_{f_n}( \sum_{i=1}^{n}  \phi_{f_n}(x_i)).
    \end{align*}
\end{thm}

\begin{cor}
    For any commutative binary operation $\circ$ defined on $G\times G$ for a finite set $G \subset \mathbb{R}$, there exists a KA representation 
\[
\phi_{<G,\circ>}:G \to \mathbb{R}^3 \quad \text{and} \quad \psi_{<G,\circ>}:\mathbb{R}^3 \to G 
\]
with the embedding dimension of $3$ such that
\begin{align*}
    x_1 \circ x_2 = \psi_{<G,\circ>}( \sum_{i=1}^2 \phi_{<G,\circ>}(x_i)).
\end{align*}
\end{cor}
In particular, 
\begin{proof}
    Define $f:G \to G$ as $f(x_1,x_2) = x_1 \circ x_2$.
    For $x \in G$, define $a_x,b_x$ as 
\[ 
a_x = \sup \{g \in G : g \le x\}, \; b_x = \inf \{g \in G : g \ge x\}.
\] 
Define $\tilde{f}:\mathbb{R}\to\mathbb{R}$  such that
\begin{align*}
 \tilde{f}(x) =  \begin{cases} 
f(b_x), \quad &\text{if} \quad a_x = -\infty, \\
f(a_x) + \frac{f(b_x)-f(a_x)}{b_x-a_x} (x-a_x), \quad &\text{if} -\infty < a_x, b_x < \infty, \\
f(a_x), \quad & \text{if} \quad b_x = \infty. 
\end{cases}
\end{align*}
Note that $\tilde{f}$ denotes the linear interpolation of $f$ between two adjacent points, $a_x$ and $b_x$ such that $a_x<x<b_x$. Since $\tilde{f}$ is continuous on $\mathbb{R}$, by Theorem \ref{thm:emb_n}, we have KA representation $\tilde{\phi}: \mathbb{R} \to \mathbb{R}^3$ and $\tilde{\psi}:\mathbb{R}^3 \to \mathbb{R}$ such that 
\[ \tilde{f}(x_1,x_2) = \tilde{\psi}(\sum_{i=1}^2\tilde{\phi}(x_i) ), \]
for $x_1,x_2 \in \mathbb{R}$. Since $f(x) = \tilde{f}(x)$ for $x \in G$, if we define $\phi_{<G,\circ>}(x) := \tilde{\phi}(x), \psi_{<G,\circ>}(x) = \tilde{\psi}(x)$ for $x \in G$, we have
\[ f(x_1,x_2)=x_1 \circ x_2 = \psi_{<G,\circ>}( \sum_{i=1}^2 \phi_{<G,\circ>}(x_i)).\]
\end{proof}

\subsection{KA representation theorem for abelian groups}
\label{subsec:abelian_group}
Note that the KA representation theorem introduced in the previous section requires an embedding dimension of $(n+1)$ with $n$ denoting the number of inputs as seen in Theorem~\ref{thm:emb_n}. In this section, we not only extend the KA representation theorem for functions defined in abelian groups but also deduce a simpler representation that requires a smaller number of embedding dimensions independent of the number of inputs $n$. Thanks to this observation, the universal representation is available regardless of the number of inputs, which eventually allows us to transfer the embedding function learned from binary operation to that of any composition of arithmetic operations. 

Let us consider an abelian group $G$ with the operation denoted by $\circ$. In this case, the expression $x_1 \circ \dots \circ x_n$ is well-defined thanks to the associativity of $\circ$. We present the corresponding KA representation for $x_1 \circ \dots \circ x_n$. The group is $\mathbb{Z}_p$  regarded as a subset of $\mathbb{R}$ via natural embedding given by $\iota : \mathbb{Z}_p \to \mathbb{R}$, $\iota(x)=x \in \mathbb{R}$. The following theorem proposes KA-type representation for finite abelian groups. Throuout the section, let us denote $X^\ast := X \setminus \{0\}$ for any set $X$.
\begin{thm}[abelian group representation]\label{thm:abel}
Let $G$ be a finite abelian group represented by
\begin{equation}\label{eq:num_gen}
G \cong C_{q_1} \times \dots C_{q_m},
\end{equation}
where $C_{q_j}$ denotes a cyclic group of order $q_j \in \mathbb{N}$, and $m$ denotes the number of generators of $G$. Then there exist $\rho_{<G,\circ>}:G\to  (\mathbb{C}^*)^m,\psi_{<G,\circ>}:(\mathbb{C}^*)^m \to G$ where 
\begin{align} 
    x_1 \circ \dots \circ x_n = \psi_{<G,\circ>}( \sum_{i=1}^n \log(\rho_{<G,\circ>}(x_i)))  = \psi_{<G,\circ>}( \sum_{i=1}^n \phi_{<G,\circ>}(x_i)), 
\end{align} 
for any $x_1, \dots , x_n \in G$, where $\phi_{<G,\circ>}(\cdot):= \log(\rho_{<G,\circ>}(\cdot))$. 
\end{thm} 
\begin{proof}
A complete proof is provided in Appendix \ref{proof_thm_abel}.
\end{proof}


To highlight the difference between Theorem~\ref{thm:emb_n} and Theorem~\ref{thm:abel}, we emphasize that the embedding and outer functions are independent of $n$, which denotes the number of inputs. Instead, they depend solely on the operation when $<G,\circ>$ forms an abelian group. This salient feature allows us to use the same embedding function obtained from the operation $x_1 \circ x_2$ for learning $x_1 \circ \cdots \circ x_n$ as they share the common embedding. More precisely, the function $f:G \times G \rightarrow G$ given by $f(x_1,x_2)=x_1 \circ x_2$ can be represented using universal embedding $\phi_{<G,\circ>}:G\to(\mathbb{C}^*)^m$ and outer function $\psi_{<G,\circ>}:(\mathbb{C}^*)^{m}\to G$ with the embedding dimension of $2m$ where $m$ denotes the number of generators of $G$ as in~\eqref{eq:num_gen}, that is,
\[
x_1 \circ x_2 = \psi_{<G,\circ>}(\sum_{i=1}^2 \phi_{<G,\circ>}(x_i)).
\]
Here, the embedding functions are ubiquitous for operations $x_1\circ\dots\circ x_n$ for any $n$, 
\[
f_n(x_1,...,x_n)=x_1\circ\dots\circ x_n = \psi_{<G,\circ>}(\sum_{i=1}^n \phi_{<G,\circ>}(x_i)).
\]
with the same $\phi$ and $\psi$ used for the binary operation.


In particular, when $m=1$ implying that $G$ is a cyclic group, we have a simpler representation for the composition of arithmetic operations, which is a direct consequence of Theorem~\ref{thm:abel}.
\begin{cor}[cyclic group representation]\label{cor:cyc}
\label{cor:cor1}
    If $G$ is a finite cyclic group, then there exists the KA representation  $\rho_{<G,\circ>}:G\to\mathbb{C}^\ast,\psi_{<G,\circ>}:\mathbb{C}^\ast \to G$ such that 
\begin{align} 
    x_1 \circ \dots \circ x_n = \psi_{<G,\circ>}( \sum_{i=1}^n \log(\rho_{<G,\circ>}(x_i)))  = \psi_{<G,\circ>}( \sum_{i=1}^n \phi_{<G,\circ>}(x_i)), 
\end{align} 
for any $x_1, \dots , x_n \in G$, where $\phi_{<G,\circ>}(\cdot) = \log(\rho_{<G,\circ>}(\cdot))$.  
\end{cor}

\begin{example}[modular addition]\label{ex:add}
    For a prime number $p$, let $\circ$ be the modular addition in $\mathbb{Z}_p$ defined as
    \begin{align*}
         x_1 \circ x_2 := x_1 + x_2 \Mod{p}.
    \end{align*}
Then $<\mathbb{Z}_p , \circ>$ becomes a cyclic group with order $p$. By Corollary \ref{cor:cor1}, there exists universal KA representation for modular addition. In fact, we can find the following KA representation with the embedding dimension of $2$:
\begin{align*}
   \rho(x)&=   \exp( \frac{2 \pi i x}{p})  \in \mathbb{C}^* \cong \mathbb{R}^2 \setminus \{0\}, \quad  x \in \mathbb{Z}_p, \\ 
   \psi(z)&=\frac{p}{2 \pi i} T(z), \quad z \in \mathbb{C}^*, 
\end{align*}
where $T:\mathbb{C}\to\mathbb{C}$ is defined as $T(a+bi)=a+\tilde{b}$ where $a,b \in \mathbb{R}, \; b = \tilde{b} + 2 \pi m_b $ for $\tilde{b} \in [0,2\pi), m_b \in \mathbb{Z}$.  
Letting $x_1 + x_2 = m_1+ m_2 p $ for $m_1 \in [0,p-1]$ and $m_2 \in\mathbb{Z}$, we have
\begin{align*}
    \psi( \sum_{i=1}^n  \phi(x_i)) &= \psi( \sum_{i=1}^n  \log(\rho(x_i) )) \\
    &= \frac{p}{2 \pi i}   T\big(   \frac{2 \pi i x_1}{p} + \frac{2 \pi i x_2}{p}   \big) \\
    &= \frac{p}{2 \pi  }  \big( \frac{2 \pi  m_1} {p} \big)\\
    &=m_1 = x_1 + x_2 \Mod{p}.
\end{align*}
Note that the representation $\phi,\psi$ are not unique. In fact, for any $k \in \mathbb{Z}$ coprime with $p$, we have another KA representing by setting
\begin{align*}
   \rho_k(x)&= \exp( \frac{2 \pi i x k }{p})  \in \mathbb{C}^* \cong \mathbb{R}^2\setminus\{0\}, \quad  x \in \mathbb{Z}_p, \\ 
   \psi_k(z)&=\frac{p}{2 \pi i k }T(z), \quad z \in \mathbb{C}^*. 
\end{align*}
\end{example}

\begin{example}[modular multiplication]\label{ex:mul}
Let $\circ$ denote the modular multiplication in $\mathbb{Z}_p$ defined as
    \begin{align*}
         x_1 \circ x_2 := x_1  \times x_2 \Mod{p}.
    \end{align*}
Since $0$ has no inverse, $\mathbb{Z}$ does not form a group with modular multiplication. We instead consider $\mathbb{Z}_p^\ast=\mathbb{Z}_p \setminus\{0\}$ for which one can find $a\in [2,p-1]$ such that $\{a,a^2,...,a^{p-1}\}=\{1,...,p-1\}$ modulo $p$. 
Hence, by Corollary \ref{cor:cor1}, there exists a KA representation for $\mathbb{Z}_p^\ast$. To illustrate, let us consider
\begin{align*}
   \rho(x)&= \exp( \frac{2 \pi i \lg_a(x)   }{p}) \in  \mathbb{C}^* \cong \mathbb{R}^2\setminus\{0\}, \quad  x \in \mathbb{Z}_p^\ast, \\ 
   \psi(z)& = \exp_a( \frac{p}{2 \pi i} T(z) ), \quad z \in \mathbb{C}^\ast, 
\end{align*}
where $\lg_a$ denotes the discrete logarithm of base $a$ in $\mathbb{Z}_p^\ast$, i.e., if $a^n = b \Mod{p}$, then $\lg_a b =n$. Let $\exp_a(x):=a^x$ denote the exponential with base $a$ in $\mathbb{Z}_p^\ast$ for brevity. For $x_1,x_2 \in \mathbb{Z}_p^\ast$, let $x_1 = a^{e_1}, x_2 = a^{e_2}$, $e_1,e_2 \in \mathbb{Z}_p^*$. Then we have 
\begin{align*}
    \sum_{i=1}^2 \phi(x_i) &=  \sum_{i=1}^2 \log(\rho(x_i)) \\
    &=  \frac{2 \pi i \lg_a(x_1)}{p} +   \frac{2 \pi i \lg_a(x_2)}{p}   \\
    &=  \frac{2 \pi i (e_1+e_2)}{p}
\end{align*}
and
\begin{align*}
    \psi(\sum_{i=1}^n \phi(x_i))&= \exp_a \big( \frac{p}{2 \pi i} T(  \frac{2 \pi i (e_1+e_2)}{p} )\big)  \\
    &= \exp_a(e_3) \\
    &= x_1 \times x_2 \Mod{p},
\end{align*}
where $e_3 = e_1 + e_2 \Mod{p}$ had $T$ is same as in Example~\ref{ex:add}.

Note that such a representation $\phi,\psi$ is not unique. For any $k \in \mathbb{Z}$ coprime with $p-1$, we have another representation by setting
\begin{align*}
   \rho_k(x)&= \exp( \frac{2 \pi i \lg_a(x) k  }{p})  \in \mathbb{C}^\ast  \cong \mathbb{R}^2 \setminus \{0\}, \quad  x \in \mathbb{Z}_p^\ast, \\ 
   \psi_k(z)&= \exp_2( \frac{p}{2 \pi i k } T(z)), \quad z \in \mathbb{C}^\ast.
\end{align*}

In addition, for $p \geq 3$, let us write $p-1 = q_1 \times q_2$ for $ q_1,q_2 (\ne 1)\in \mathbb{N}$. With $x = a^{e_x}, e_x \in \mathbb{Z}_p$ and   $e_x  = b_x q_1 + r_x$ for $r \in [0,q_1-1]$ and $b \in \mathbb{Z}$, we have another KA representation given as
\begin{align*}
   \rho_k(x)&= \big( \exp( \frac{2 \pi i b_x k_1  }{q_1}), \exp( \frac{2 \pi i r_x k_2  }{q_2}) \big)  \in (\mathbb{C}^\ast)^2, \quad  x \in \mathbb{Z}_p, \\ 
   \psi_k(z)&= \exp_a( \frac{p}{2 \pi i k_1 } w_1 )\times q_1 + \exp_a( \frac{p}{2 \pi i k_2 } w_2 ) , \quad  (z_1,z_2) \in (\mathbb{C}^*)^2, 
\end{align*}
where $(w_1,w_2) = T_2(z_1,z_2), T_2:\mathbb{C}^2\to\mathbb{C}^2$ is defined as 
\[
T_2(z_1,z_2)=(T(z_1+R(z_2)), T(z_2))\quad\text{and}\quad R(a+bi) = 2 \pi m_b i, 
\]
where $a,b \in \mathbb{R}, \; b = \tilde{b} + 2 \pi m_b $ for $\tilde{b} \in [0,2 \pi)$, $m_b \in \mathbb{Z}$.  
\end{example}

\subsection{KA representation theorem for an anti-abelian operation}
We propose that the KA representation theorem can be further extended to anti-abelian operations. We say an operation $\bullet$ is anti-abelian for the abelian group $<G,\circ>$, if $x_1 \bullet x_2 = x_1 \circ x_2^{-1}$ where $ x_2^{-1}$ is an inverse of $x_2$. Let us set $\phi$ and $\psi$ be embedding and outer functions corresponding to the binary operation $\circ$. Then, we have the following that provides evidence for the effectiveness of leveraging weight transfer.
\begin{thm}
[anti-abelian operation]\label{thm:anti_abel}
Let $\bullet$ be an anti-abelian binary operation for the abelian group operation $\circ$. Let $m$ denote the number of generators of $G$. 
For any  $\ast$, there exist
\[
\phi_{<G,\circ>}:G\to(\mathbb{C}^{\ast})^m \quad \text{and} \quad \psi_{<G,\circ>}:(\mathbb{C}^{\ast})^m\to G, 
\]
such that 
\begin{align*} 
    x_1 \bullet x_2   &= \psi_{<G,\circ>} \big(\log(\rho_{<G,\circ>}(x_1)) - \log(\rho_{<G,\circ>}(x_2)) \big)\\
    &=  \psi_{<G,\circ>} \big( \phi_{<G,\circ>}(x_1) - \phi_{<G,\circ>}(x_2) \big),
\end{align*} 
where $\phi_{<G,\circ>}(\cdot) =\log(\rho_{<G,\circ>}(\cdot))$.     
\end{thm}
\begin{proof}
In this proof, we follow the notations from the proof of Theorem \ref{thm:abel}. 
Let $x_3 := x_1 \bullet x_2$. Then we have $x_2 \circ x_3 = x_1$. Since $\circ$ is an abelian group operation, by Theorem \ref{thm:abel}, there exist $\rho_{<G,\circ>},\phi_{<G,\circ>} : G \to (\mathbb{C}^\ast)^m$, $\psi_{<G,\circ>} : (\mathbb{C}^\ast)^m \to G$ and $T:\mathbb{C} \to \mathbb{C}$  such that
\begin{align*}
    \psi_{<G,\circ>}  \big( \log(\rho_{<G,\circ>}(x_2)) + \log(\rho_{<G,\circ>}(x_3)) \big) = x_1,
\end{align*}
where $\rho_{<G,\circ>},\phi_{<G,\circ>},\psi_{<G,\circ>},T$ are defined in the proof of Theorem \ref{thm:abel}. Since $\psi_{<G,\circ>}(\cdot) = \rho_{<G,\circ>}^{-1}(\exp(T(\cdot)))$, we have 
\begin{align*}
    &\psi_{<G,\circ>}  \big( \log(\rho_{<G,\circ>}(x_2)) + \log(\rho_{<G,\circ>}(x_3)) \big) \\
    &= \rho_{<G,\circ>}^{-1} \bigg( \exp \bigg( T \bigg( \log(\rho_{<G,\circ>}(x_2)) + \log(\rho_{<G,\circ>}(x_3)) \bigg) \bigg)\bigg)\\
    &= x_1, \\
     &T \big( \log(\rho_{<G,\circ>}(x_2)) + \log(\rho_{<G,\circ>}(x_3)) \big) = \log(x_1), 
\end{align*}
and
\begin{align*}
     &\log(\rho_{<G,\circ>}(x_2)) + \log(\rho_{<G,\circ>}(x_3))  = \log(\rho_{<G,\circ>}(x_1)) + (2 \pi i) m_1  
\end{align*}
for some $m_1 \in \mathbb{Z}$.
Therefore, we have 
\begin{align*}
    \log(\rho_{<G,\circ>}(x_1)) - \log(\rho_{<G,\circ>}(x_2))  = \log(\rho_{<G,\circ>}(x_3))  - (2 \pi i) m_1,
\end{align*}
leading to
\begin{align*}
    \psi \big( \log(\rho_{<G,\circ>}(x_1)) - \log(\rho_{<G,\circ>}(x_2)) \big) = x_3.  
\end{align*}
\end{proof}
\subsection{KA representation and transformer}
We discuss the KA representations of basic arithmetic operations and how they are related. Throughout the section, we suppress $G$ in embedding and outer functions for brevity. First, given arithmetic operation $\circ$, we have the following KA representation
\[   x_1 \circ \dots \circ x_n =  \psi_\circ ( \sum_{i=1}^n \lambda_i \phi_\circ (x_i)), \]
where $\phi_\circ: G \to \mathbb{R}^d$, $\lambda_i \in \mathbb{R}$, $\psi_\circ: \mathbb{R}^d \to G$ with the embedding dimension $d \in \mathbb{N}$. 
We decompose the above representation into three parts: $\Phi$, $\Sigma$, and $\Psi$ where
\begin{align*}
    &\Phi(x_1 , \dots , x_n) = [\phi_\circ(x_1), \dots , \phi_\circ(x_n)], \quad \phi(x_i) \in \mathbb{R}^d, \\
    &\Sigma(\textbf{z}_1 , \dots , \textbf{z}_n) = \sum_{i=1}^n \lambda_i \textbf{z}_i \in \mathbb{R}^d, \\
    &\Psi( \textbf{z}) = \psi_\circ(\textbf{z}) \in G
\end{align*}
for $\textbf{z}_1 , \dots , \textbf{z}_n, \textbf{z} \in \mathbb{R}^d$.
Then we have the following decomposition of KA representation:
\[   x_1 \circ \dots \circ x_n =  \psi_\circ ( \sum_{i=1}^n \lambda_i \phi_\circ (x_i)) = \Psi(\Sigma(\Phi(x_1, \dots , x_n))). \]

Now we explore which parts of the KA representation from various basic arithmetic operations are shared when equipped with algebraic structures, commutativity, abelian, and anti-abelian. Throughout the section, we will consider such operations defined in $\mathbb{Z}_p$.

\subsubsection{Sharing $\Sigma$ - commutative operations and anti-abelian}
Recalling Theorem~\ref{thm:emb_n}, we see that any commutative binary operation shares the common $\sum$. In particular, addition and multiplication can be expressed as
\[
x_1 + x_2 = \psi_+(\sum_{i=1}^2 \phi_+(x_i)), \quad x_1 \times x_2 = \psi_\times (\sum_{i=1}^2 \phi_\times(x_i)).
\]
Therefore, $x_1 + x_2$ and $x_1 \times x_2$ have common $\Sigma$ part with 
\[
\Sigma(\textbf{z}_1 , \textbf{z}_2 ) = \textbf{z}_1 + \textbf{z}_2.
\]

We extend the argument for the case of anti-abelian. Since subtraction and division are anti-abelian for addition and multiplication respectively they have the following representations by Theorem~\ref{thm:anti_abel}:
\[
x_1 - x_2 = \psi_+(\phi_+(x_1)- \phi_+(x_2)), \quad x_1 / x_2 = \psi_\times (\sum_{i=1}^2 \phi_\times(x_1) - \phi_\times(x_2)).
\]
Therefore, $x_1 - x_2$ and $x_1 / x_2$ have common $\Sigma$ part with 
\[
\Sigma(\textbf{z}_1 , \textbf{z}_2 ) = \textbf{z}_1 - \textbf{z}_2.
\]
In the following section, we continue our discussion with abelian groups.
\subsubsection{Sharing $\Phi$ - abelian}\label{sec:sharing_phi}
As discussed in Section~\ref{subsec:abelian_group} and Corollary \ref{cor:cyc}, if $<G,\circ>$ forms a finite abelian group, there exists ubiquitous embedding function $\phi_\circ$ and $\psi_\circ$, independent of the number of operands $n$, such that 
\[
x_1\circ\dots\circ x_n = \psi_\circ(\sum_{i=1}^n \phi_\circ(x_i)).
\]
Since both $<\mathbb{Z}_p,+>$ and $<\mathbb{Z}^*_p,\times>$ with $p\geq 2$ form abelian groups, we deduce that $x_1 + x_2$, $ x_1+x_2+x_3$, $x_1 + \dots + x_n$ share the common embedding function $\phi_+$. By the same argument, we argue that $x_1 \times x_2$, $x_1\times x_2 \times x_3$, $x_1 \times \dots \times x_n$ share the common embedding function $\phi_\times$.

\subsubsection{Leveraging shared part of KA representation}
The overall landscape for shared structure among arithmetic operations through the lens of representations is demonstrated in Figure~\ref{fig:KA}. Based on the observation, we propose various transfer learning approaches to leverage the sharing of representations within different operations, which is discussed in detail in the following section.
\begin{figure}[H]\centering
   \includegraphics[width=0.7\textwidth]{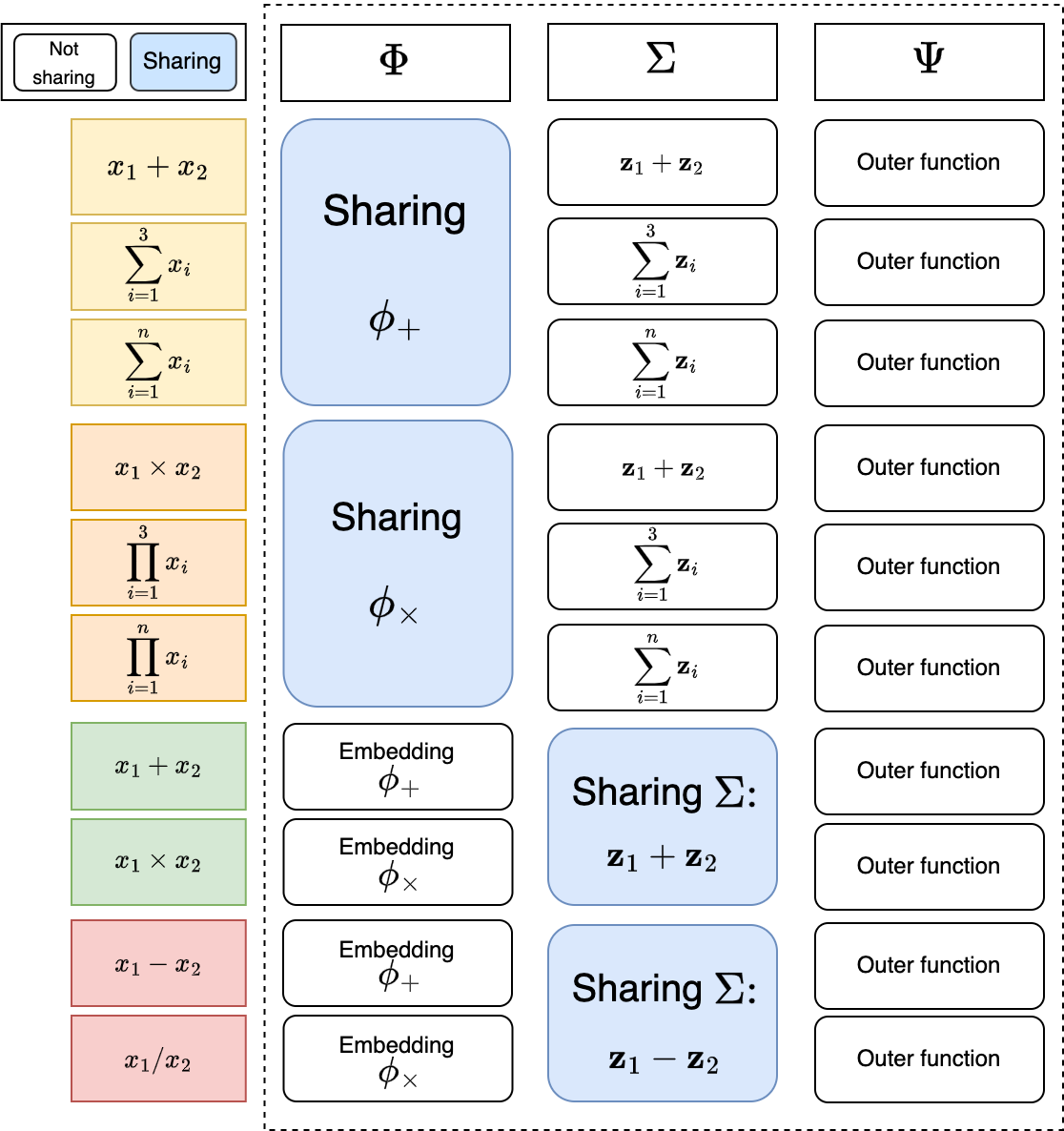} 
    \caption{Illustration of shared structures within KA representation for binary operations.}
    \label{fig:KA}
\end{figure}

\section{Methodology}
Focusing on arithmetic operations that include four elementary operations addition, subtraction, multiplication, and division as well as combinations of them, we propose a new method for learning various operations efficiently by taking algebraic properties of operations into account. Let us begin by introducing the data augmentation technique for learning commutative operation. When the operation is commutative, we add all commutative pairs into the training batch, which improves the training efficiency. The detailed procedure is demonstrated in Section~\ref{sec:method:aug}, and its empirical evidence is provided in the next section.

Leveraging the KA representation theorem, weight parameters of a decoder block or embedding layer that are learned from operations can be transferred for efficient learning of new operations as illustrated in Section~\ref{sec:method:weight}. Furthermore, the KA representation theorem for the finite abelian group given in Theorem~\ref{thm:abel} motivates us to transfer weights of the embedding layer to learn the composition of arithmetic operations, thereby acceleration of grokking is observed as presented in Section~\ref{sec:exp:wegiht_transfer}. We also argue that

Lastly, we also propose a new approach for learning arithmetic operations by introducing relational equations.

\subsection{Learning arithmetic operations with language model}\label{sec:aug}
Central to learning arithmetic operation $ a \circ b = c$ is to take ['$a$', '$\circ$', '$b$','$=$'] as an input of the transformer, and trains the model to predict the output $c$ accurately. We also extend our learning mechanism to predict the output of the composition of arithmetic operations, that is, $a_1 \circ a_2 \circ ... \circ a_n$.

Given a set $G$ and binary operation 
'$\circ$' associated with it, our task is to learn neural network model $c_i \approx f_\theta(a_i,b_i)$ where $\theta$ represents the neural network parameter. To this end, we randomly split the set $\{c:c=a \circ b, (a,b)\in G\times G\}$ into training and test set and train the model on the training set to predict the output. As noted, such a task often suffers from grokking when the training set is small, we propose a way of encompassing the algebraic nature of the binary relation to accelerate it, which is referred to \textit{degrokking}~\cite{tan2023understanding}.

\subsection{Augmentation for learning commutative binary operation}\label{sec:method:aug}
When the binary operation is commutative, we can always augment the dataset by adding commutative pairs. To this end, for all $(a,b) \in \mathcal{B}_{\text{train}}$, we update $\mathcal{B}_{\text{train}}\leftarrow \mathcal{B}_{\text{train}} \cup \{(b,a)\}$ during the training. Our method differs from that proposed by~\cite{tan2023understanding} as they modify the loss function by adding commutative pair loss. 

Our simple augmentation incentivizes the model to learn the commutative structure, and hence, acceleration of grokking is achieved. As mentioned, authors of ~\cite{tan2023understanding} introduce the difference between to commutative pairs to force the model to learn the intrinsic structure. We generalize the idea by simply adding commutative pairs of batch data, which also can be used to learn the composition of binary operations.

\subsection{Learning arithmetic operations via weight transfer}\label{sec:method:weight}

\begin{figure}[H]\centering
   \includegraphics[width=0.7\textwidth]{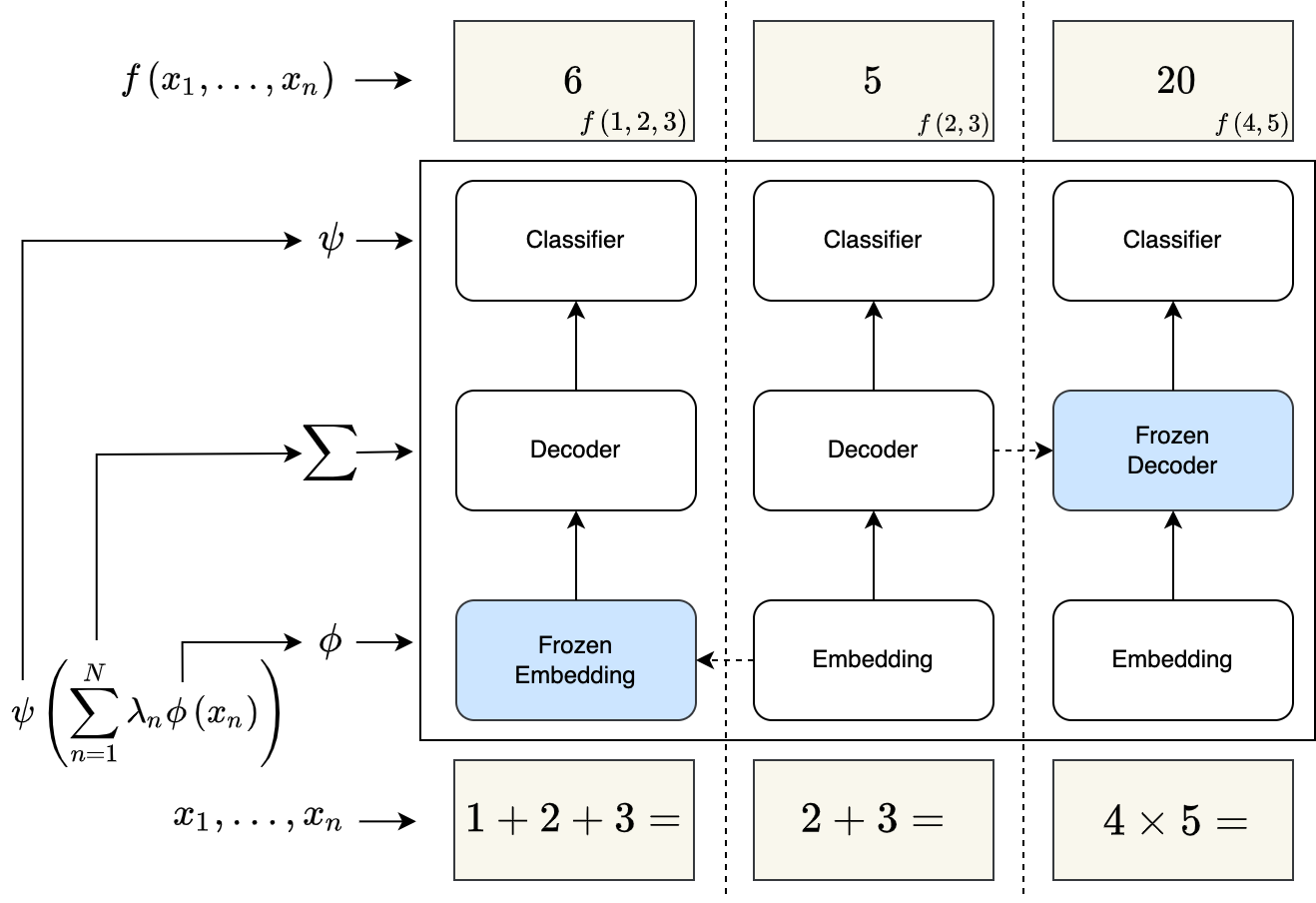} 
    \caption{Schematic diagram for weight transfer. Weights for learning addition (middle) are transferred to learn the composited operations (left) and other commutative binary operations (right).}
    \label{fig:transfer}
\end{figure}
We propose two types of weight transfer for learning binary operations efficiently, decoder block transfer and embedding transfer. 

\noindent\textbf{Decoder block transfer}: Thanks to Theorem~\ref{thm:abel}, we propose to learn commutative binary operation by transferring the weight of the decoder block learned from another commutative binary operation as any commutative binary operations have the same structural properties, $\lambda_1=\lambda_2=1$ and $n=2$. For justification, we first learn the whole transformer architecture with a commutative binary operation and transfer the weights of the decoder block. For learning a different commutative binary operation, the transferred decoder block is frozen, and only the embedding and classifier layer are learned. A similar perspective leads to a refined learning framework for an anti-abelian operation thanks to Theorem~\ref{thm:anti_abel}.  

\noindent\textbf{Embedding transfer}: When $<G,\circ>$ is equipped with a further structural property, $G$ is a finite abelian group with respect to the operation $\circ$, Theorem~\ref{thm:abel} yields that embedding dimension is fixed as $2m$ for $m=|G|$ no matter how many times the operation is composited. From this observation, one can learn the composition of a binary operation by transferring the weight of the embedding layer learned from the binary operation. An explicit flowchart is given in Figure~\ref{fig:transfer:system}.

\noindent\textbf{Learning a system of equations with unknowns:} Additionally, we design a nonstandard task where multiple unknowns in a system of equations are learned. To be specific, we aim to solve equations with unknown tokens $A$ and $B$ given $a,b,c \in G$ through the relation
\[
\begin{cases}
 &a \circ b = A \\
 &A \circ c = B \quad \text{ask}\quad B,
\end{cases}
\]
or
\[
\begin{cases}
&a \circ A = b \\
&A \circ c = B  \quad \text{ask}\quad A,
\end{cases}
\]
which can be used as inputs in the form of
\[
a \circ b = A \& A \circ c = B \quad B? \quad \text{and} \quad a \circ A = b \&A \circ c = B  \quad A?.
\]

Unlike the previous task, where we only predicted the output of an arithmetic operation, learning unknowns in a system of equations is much more complicated. We observe that solving unknowns in a system of equations is similar to solving the composition of operators. Referring to Section \ref{sec:sharing_phi}, we conjecture that the embedding layer in learning a system of equations involving an abelian operation inherits properties from the embedding used for the abelian operation itself.

Hence, we propose the idea of transferring the embedding trained with the binary operation. Unlike the task of the composition of operations, we introduce new tokens such as unknowns $A, B$ and the 'and' token denoted by $'\&'$. Consequently, it is not reasonable to freeze the entire embedding layer because these new tokens require their own representations. Instead, we argue for a hybrid approach that combines the transferred embedding layer with a new, trainable embedding layer to preserve the knowledge from the binary operation while accommodating new information. This concept is explicitly illustrated in Figure~\ref{fig:transfer}, which demonstrates how the transferred and new embedding layers interact to solve systems of equations with unknowns.

\begin{figure}[H]
\centering
   \includegraphics[width=0.7\textwidth]{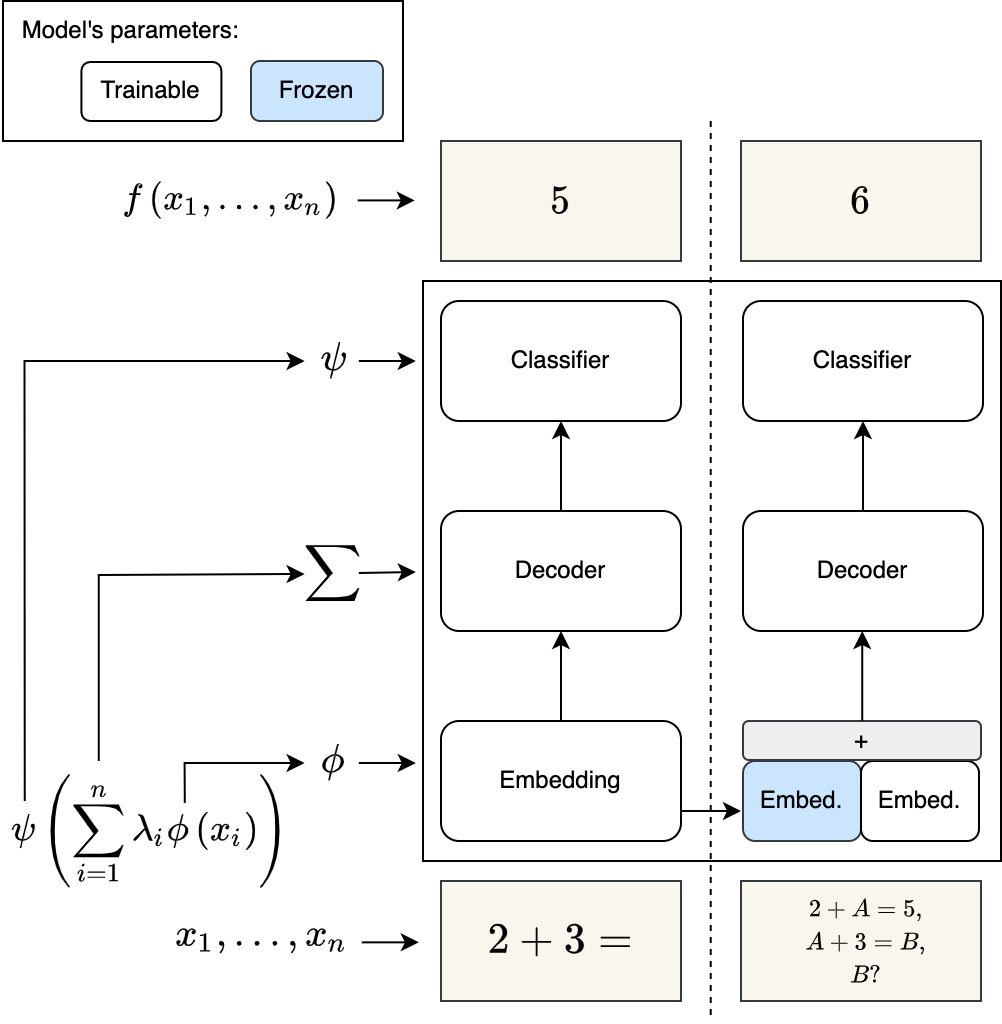} 
    \caption{Schematic diagram for learning a system of equations via weight transfer.}
    \label{fig:transfer:system}
\end{figure}

\section{Experiments}\label{sec:exp}

\subsection{Dataset and experiment setup}
As in the study~\cite{power2022grokking}, we train the transformer using an algorithmic dataset of binary relations defined on $\mathbb{Z}_p \times \mathbb{Z}_p$ for some prime number $p$. A comprehensive set of experiments is carried out, including (i) commutative data augmentation, (ii) decoder block transfer, (iii) embedding layer transfer, and (iv) implementation of a limited set of tokenization methods. We set $p=97$ for all experiments.

\noindent\textbf{Arithmetic binary operation:} For training, we use $x_1 \circ x_2$ where $x_1,x_2\in\mathbb{Z}_p$ as inputs and let model learn $x_3:= x_1 \circ x_2$. We consider various binary operations $x_1 \circ x_2$ given as
\begin{equation*}
\begin{split}
& x_1+x_2,\; x_1 \times x_2,\; x_1^2 + x_2^2,\; x_1^2+x_2^2+ x_1 +x_2,\; 
 x_1^3+x_2^3+x_1+x_2,\\
& x_1^2+x_2^2+x_1x_2,\; x_1- x_2,\; x_1/x_2, (x_1+x_2)^2,\; (x_1+x_2)^3.
\end{split}
\end{equation*}

\noindent\textbf{Composition of arithmetic operations:} Another task we are interested in is to learn $n \in \mathbb{N}$ compositions of a given binary operation, 
\[
x_1 \circ \cdots \circ x_n
\]
for $x_1,...,x_n \in \mathbb{Z}_p$. For validation of this task, we randomly sample $30,000$ data samples, and we also test if grokking has occurred with $100,000$ random data samples. 


\noindent\textbf{A system of equations with unknowns:} We solve a system of equations with two unknowns taking the form of 
\[
\begin{cases}
 &2 + A = 5 \\
 &A + 4 = B 
\end{cases}
\quad \text{ask} \quad B\quad \text{and} \quad 
\begin{cases}
 &1 + 2 = A \\
 &A + 3 = B 
\end{cases}
\quad \text{ask} \quad A,
\]
or
\[
\begin{cases}
 &2 \times A = 6 \\
 &A \times 4 = B 
\end{cases}
\quad \text{ask} \quad B
\quad \text{and} \quad 
\begin{cases}
 &1 \times 2 = A \\
 &A \times 3 = B 
\end{cases}
\quad \text{ask} \quad A.
\]
For validation, we again randomly sample $30,000$ data samples, and we also test if grokking has occurred with $100,000$ random data samples as above.

\noindent\textbf{Learning with a limited number of tokens:} We aim to train the model using a limited set of numbers instead of using all tokens in $\mathbb{Z}_p$ for the composition of arithmetic operations. Due to the nature of KA representation, we transfer the embedding layer learned with some arithmetic operations and learn $\Sigma$ only with a limited set of tokens to verify that the model indeed learns the pattern of operations rather than memorizing. We also test the approach using 100,000 random samples to determine whether grokking occurs.
\subsubsection{Model description}
All experiments utilize the transformer that consists of an embedding layer with four MLP modules, two decoder blocks, each with four attention heads, and an MLP classifier layer. Both the embedding layers and the attention dimensions are set to 256. 


\subsubsection{Training details}
For the size of the mini-batch, we leverage 1024 for learning binary operations and a system of equations with unknowns, and 4096 for the composition of operations. A weight decay is applied through the AdamW optimizer \cite{kingma2014adam,loshchilov2017decoupled} with a learning set and a decay rate of $0.001$ and $0.1$ respectively. We regard that the model is grokked when the test accuracy exceeds $99 \%$ within $10^5$ steps.

\noindent\textbf{Visualization of Embedding:} To visualize the correspondence between KA representation and the learned transformer for arithmetic operations, we present the visualization of the embedding of tokens. After learning addition and multiplication, the first two principal components of the embedding vectors for each token in $\mathbb{Z}_p$ are illustrated in Figure~\ref{fig:PCA}.

Regarding addition, we observe that the embedding vectors form a circle, which aligns with the KA representation given in Example~\ref{ex:add}. Denoting the angle between two tokens $i$ and $j$ by $\theta_{i,j}$, it is notable that the $\theta_{i,i+1}$ values are similar for all $i$, indicating that the embedding reflects the group structure of addition.

Similarly, the embedding vectors obtained from learning multiplication also form a circle when projected into 2-dimensional space via principal component analysis (PCA). As discussed in Example~\ref{ex:mul}, since $\mathbb{Z}_p \setminus {0}$ forms an abelian group, this representation also aligns with our observations.

These empirical results support the idea of a correspondence between the embedding function in the KA representation theorem and the embedding layer in the transformer architecture. Therefore, we expect that the decoder block functions as the intermediate summation part ($\Sigma$), underscoring the validity of the weight transfer method described in Section \ref{sec:method:weight}.



\begin{figure}[H]
\centering
   \includegraphics[width=0.7\textwidth]{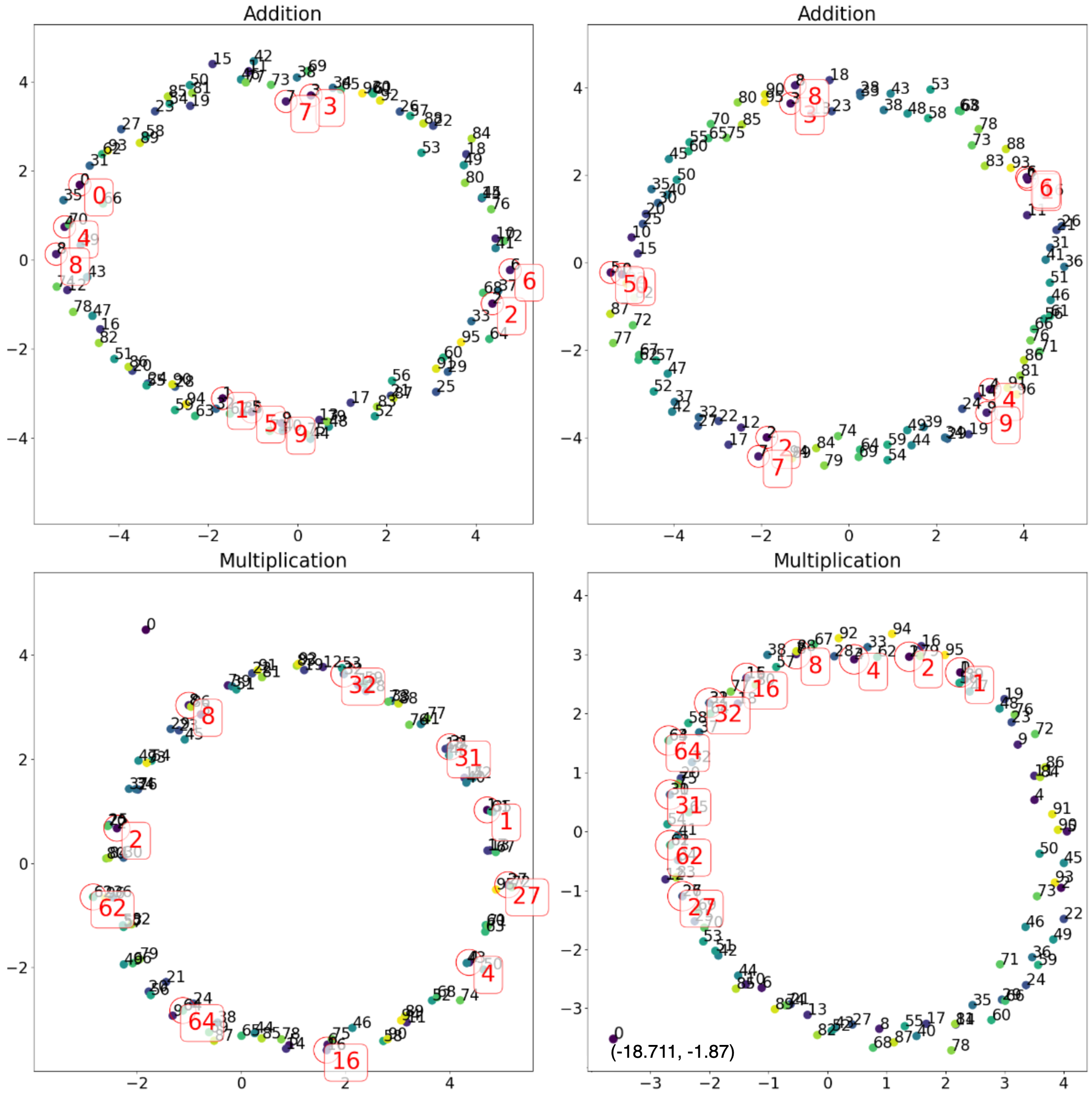} 
    \caption{PCA visualization for addition (top) and multiplication (bottom). We see that $\theta_{i,i+1}$'s, the angles between token $i$ and $i+1$, are maintained similarly. A similar feature is observed in the multiplication hinting that tokens form a multiplicative group. }
    \label{fig:PCA}
\end{figure}

\subsection{Learning various commutative arithmetic binary operations via augmentation}\label{sec:exp:aug}
As illustrated in Section~\ref{sec:method:aug}, we add commutative pairs during the training and compared the validation and test accuracy. We examine how increased exposure to commutative information affects the grokking behavior for various commutative arithmetic operations by comparing the initial optimization steps at which grokking occurs.

\begin{table}[H]
\caption{Number of steps required to achieve grokking via commutative augmentation (CA). Grokking is tested with full data besides training data. Here, $N$ denotes the number of training samples.}
\vspace{0.3cm}
\centering
\scalebox{0.6} {\begin{tabular}{@{}ccccc@{}}
\toprule
Operation & $N$ & Method & Grokking step & Final accuracy \\
\midrule
\multirow{6}{*}{$x_1+x_2$} 
&  \multirow{2}{*}{5000} 
 & baseline & 27167($\pm$ 28853) & Grokked \\
 &  & +CA  & 10171($\pm$ 7561) & Grokked \\
 \cmidrule{2-5}
&  \multirow{2}{*}{4000} 
 & baseline & 15359($\pm$ 0) & Grokked \\
 &  & CA  & 31699($\pm$ 0) & Grokked \\
  \cmidrule{2-5}
 &  \multirow{2}{*}{3000} 
 & baseline & Non-grokked & 33.53($\pm$ 0) \\
 &  & +CA  & Non-grokked & 53.14($\pm$ 0) \\

\midrule
\multirow{6}{*}{$x_1 \times x_2$} 
&  \multirow{2}{*}{5000} 
 & baseline & 32839($\pm$ 0) & Grokked \\
 &  & +CA  & 7419($\pm$ 0) & Grokked \\
  \cmidrule{2-5}
&  \multirow{2}{*}{4000} 
 & baseline & Non-grokked & 59.88($\pm$ 0)\\
 &  & +CA  & 31419($\pm$ 0) & Grokked \\
 \cmidrule{2-5}
 &  \multirow{2}{*}{3000} 
 & baseline & Non-grokked & 35.30($\pm$ 0) \\
 &  & +CA  & Non-grokked & 54.23($\pm$ 0) \\

\midrule
\multirow{6}{*}{$x_1^{2}+x_2^{2}$ $+x_1+x_2$} 
&  \multirow{2}{*}{6000} 
 & baseline & 584($\pm$ 199) & Grokked \\
 &  & +CA  & 714($\pm$ 111) & Grokked \\
 \cmidrule{2-5}
&  \multirow{2}{*}{5000} 
 & baseline & 3374($\pm$ 3110) & Grokked \\
 &  & +CA  & 1014($\pm$ 332) & Grokked \\
 \cmidrule{2-5}
 &  \multirow{2}{*}{4000} 
 & baseline & Non-grokked & 98.61($\pm$ 0.03) \\
 &  & +CA  & 63549($\pm$ 41903) & Grokked \\

\midrule
\multirow{6}{*}{$x_1^{3}+x_2^{3}$ $+x_1+x_2$} 
 &  \multirow{2}{*}{5000} 
 & baseline & 6304($\pm$ 6622) & Grokked \\
 &  & +CA  & 4144($\pm$ 6274) & Grokked \\
 \cmidrule{2-5}
&  \multirow{2}{*}{4000} 
 & baseline & 12929($\pm$ 13511) & Grokked \\
 &  & +CA  & 7129($\pm$ 7611) & Grokked \\
  \cmidrule{2-5}
&  \multirow{2}{*}{3000} 
 & baseline & 48769($\pm$ 34298) & Grokked \\
 &  & +CA  & 33879($\pm$ 28652) & Grokked \\

\midrule
\multirow{6}{*}{$x_1^{2}+x_2^{2}$ $+x_1x_2$} 
&  \multirow{2}{*}{9000} 
 & baseline & 1859($\pm$ 84) & Grokked \\
 &  & +CA  & 2989($\pm$ 296) & Grokked \\

 \cmidrule{2-5}
&  \multirow{2}{*}{8000} 
 & baseline & Not-grokked & 97.43($\pm$ 0.20)\\
 &  & +CA  & Not-grokked & 97.27($\pm$ 0.20) \\

 \cmidrule{2-5}
 &  \multirow{2}{*}{7000} 
 & baseline & Not-grokked & 93.23($\pm$ 0.05) \\
 &  & +CA  & Not-grokked & 92.74($\pm$ 0.55) \\

 \toprule
\end{tabular}
}
\label{table:CA}
\end{table}

\subsection{Learning arithmetic operations via weight transfer}\label{sec:exp:wegiht_transfer}
In this section, we argue that transfer learning is indeed efficient for the acceleration of grokking as supported by the KA representation theorem. Depending on tasks and the group structures, either the decoder block or embedding layer is frozen.

\subsubsection{Decoder block transfer}
\noindent\textbf{Commutative operation to commutative operation:} As noted, the role of the decoder block is interpreted as the summation of embedding function with $\lambda_1=\lambda_2=1$ when learned from a commutative binary operation. To support our claim, we transfer the weights of decoder block corresponding to $x_1+x_2$ to accelerate the learning of $x_1\times x_2$, $(x_1+x_2)^2$, $(x_1+x_2)^3$, $x_1^2+x_2^2+x_1+x_2$, and $x_1^3+x_2^3+x_1+x_2$. Additionally, we verify that such a weight transfer is effective for learning $x_1^2+x_2^2+x_1+x_2$ and $x_1+x_2$ from $x_1^2+x_2^2$ and $x_1\times x_2$. 

\noindent\textbf{Anti-abelian operation:} By the same argument, we first learn $x_1-x_2$ and $x_1/x_2$, and transfer the decoder block to learn  $x_1/x_2$ and $x_1-x_2$ respectively. Experimental results supporting our claim are presented in Table~\ref{table:DT}.

\begin{figure}[H]
\centering
   \includegraphics[width=.8\textwidth]{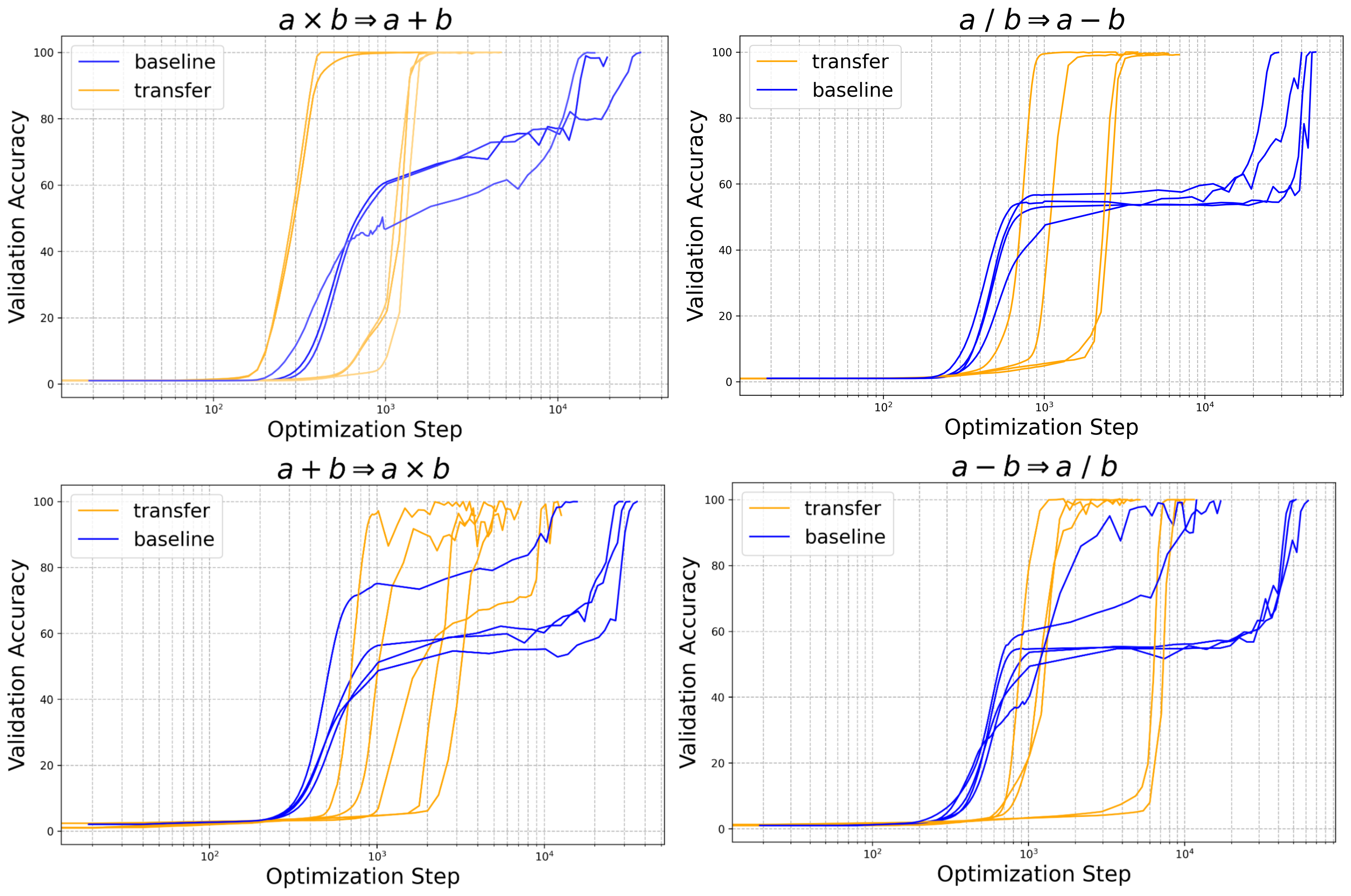} 
    \caption{The plots show the effectiveness of decoder block transfer in accelerating grokking compared to the baseline.}
    \label{fig:grok:decoder}
\end{figure}

\begin{table}[h]
\caption{The number of steps required to achieve grokking via decoder block transfer (DT) is tested using the full dataset, excluding the training data. Here, $N$ denotes the number of training samples.}
\vspace{0.3cm}
\centering
\scalebox{0.6}{\begin{tabular}{@{}ccccc@{}}
\toprule
Operation & $N$ & Method & Grokking step & Final accuracy \\
\midrule
\multirow{2}{*}{$x_1 \times x_2$} &  \multirow{2}{*}{5000} & baseline     & 20367($\pm$ 13414) & \multirow{2}{*}{Grokked} \\
 &  &  DT from $x+y$ & 3939($\pm$ 3421) &  \\
 \cmidrule{1-5}
\multirow{2}{*}{$x_1 / x_2$} &  \multirow{2}{*}{5000} & baseline     & 35543($\pm$ 22517) & \multirow{2}{*}{Grokked} \\
 &  &  DT from $x-y$ & 4115($\pm$ 3452)  &  \\
 \cmidrule{1-5}
\multirow{2}{*}{$x_1 + x_2$} &  \multirow{2}{*}{5000} & baseline     & 11647($\pm$ 10705) & \multirow{2}{*}{Grokked} \\
 &  &  DT from $x_1 \times x_2$  & 1087($\pm$ 642) & \\
 \cmidrule{1-5}
\multirow{2}{*}{$x_1 - x_2$} &  \multirow{2}{*}{5000} & baseline     & 38879($\pm$ 8081) & \multirow{2}{*}{Grokked} \\
 &  &  DT from $x_1/x_2$ & 2115($\pm$ 1288) &  \\
 \cmidrule{1-5}
 \multirow{2}{*}{$(x_1+x_2)^2$} &  \multirow{2}{*}{5000} & baseline     & 28459($\pm$ 9470) & \multirow{2}{*}{Grokked} \\
 &  &  DT from $x_1 + x_2$ & 2635($\pm$ 4666) &  \\
 \cmidrule{1-5}
  \multirow{2}{*}{$(x_1+x_2)^3$} &  \multirow{2}{*}{5000} & baseline     & Non-grokked & 77.01($\pm$1.74) \\
 &  &  DT from $x_1 + x_2$ & 479($\pm$ 107) & Grokked   \\
 \cmidrule{1-5}
   \multirow{2}{*}{$x_1^2 + x_2^2 + x_1 + x_2$} &  \multirow{2}{*}{5000} & baseline     & 755($\pm$ 293) & \multirow{2}{*}{Grokked}  \\
 &  &  DT from $x_1^2 + x_2^2$ & 359($\pm$ 50) &  \\
\toprule
\end{tabular}
}
\label{table:DT}
\end{table}

\subsubsection{Embedding transfer}\label{sec:exp:embedding}
\noindent\textbf{Composition of arithmetic operations:} After learning $x_1+x_2$ the weight of the embedding layer is transferred to accelerate the learning of $x_1+x_2+x_3$, $x_1+x_2+x_3+x_4$. Similarly, weights achieved with $x_1 \times x_2$ are transferred to learn $x_1 \times x_2 \times x_3$ and $x_1 \times x_2 \times x_3 \times x_4$ as they share the embedding function as discussed earlier. 
The reduced number of grokking steps are found in Table~\ref{table:nary}, which supports that embedding layer transfer accelerate the grokking. 

\begin{table}[H]
\caption{Number of steps required to achieve grokking with 100,000 random samples for learning the composition of arithmetic operations via embedding transfer (ET).
}
\vspace{0.3cm}
\centering
\scalebox{0.6}{\begin{tabular}{@{}ccccc@{}}
\toprule
Operation & $N$ & Method & Grokking step & Final accuracy \\
\midrule
 \multirow{6}{*}{$x_1 + x_2 + x_3$} 
  &  \multirow{2}{*}{10000} & baseline & Non-grokked & 32.17($\pm$ 11.01) \\
  &                         &  ET       &   29491($\pm$ 22414)         & Grokked \\ 
  \cmidrule{2-5}
  &  \multirow{2}{*}{100000} & baseline & 22186($\pm$ 25305) & 74.98($\pm$ 22.92)\\
    &                         &  ET       &    3039($\pm$ 5478)        & Grokked\\ 
    \cmidrule{2-5}
  &  \multirow{2}{*}{300000} & baseline & 42018($\pm$ 13873) & 80.70($\pm$ 19.20)\\
  &                         &  ET       &   10303($\pm$ 10661)         &  Grokked \\
\midrule
 \multirow{6}{*}{$x_1 \times x_2 \times x_3$} 
  &  \multirow{2}{*}{10000} & baseline & Non-grokked & 25.74($\pm$ 6.81)\\
  &                         &  ET       &     11867($\pm$ 7080)     & Grokked \\ 
  \cmidrule{2-5}
  &  \multirow{2}{*}{100000} & baseline & Non-grokked & 54.91($\pm$ 10.94)\\
    &                         &  ET       &   1087($\pm$ 380)         & Grokked \\
    \cmidrule{2-5}
  &  \multirow{2}{*}{300000} & baseline & 36852($\pm$ 46539) & 78.78($\pm$ 25.56)\\
  &                         &  ET       &   2591($\pm$ 407)         &     Grokked \\
  \midrule
 \multirow{2}{*}{$x_1 + x_2 + x_3 + x_4 $} 
  &  \multirow{2}{*}{100000} & baseline & Non-grokked & 63.62\\
  &                         &  ET       &     3559     & Grokked \\ 
  \midrule
 \multirow{2}{*}{$x_1 \times x_2 \times x_3 \times x_4 $} 
  &  \multirow{2}{*}{100000} & baseline & Non-grokked & 67.11\\
  &                         &  ET       &     1737     & Grokked \\ 
\toprule
\end{tabular} }
\label{table:nary}
\end{table}

\noindent\textbf{A system of equations with unknowns:} As shown in Figure~\ref{fig:grok:embedding} and Table~\ref{table:embedding}, embedding transfer leads to a significant reduction in grokking steps for all cases considered. 
\begin{table}[H]
\caption{Number of steps required to achieve grokking with 100,000 random samples for learning a system of equations via embedding transfer (ET).}
\vspace{0.3cm}
\centering
\scalebox{0.6}{
\begin{tabular}{@{}ccccc@{}}
\toprule
Operation & $N$ & Method & Grokking step & Final accuracy\\
\midrule
\multirow{4}{*}{$x_1+x_2$} 
 &  \multirow{2}{*}{100000} & baseline     & Non-grokked & 91.86($\pm$ 3.62) \\
 &  & ET from $x_1+x_2$  & 41615($\pm$ 29826) & Grokked \\
 \cmidrule{2-5}
&  \multirow{2}{*}{50000} & baseline     & Non-grokked & 85.60($\pm$ 2.47) \\
 &  & ET from $x_1+x_2$ & 3075($\pm$ 1853) & Grokked \\
\midrule
 \multirow{4}{*}{$x_1\times x_2$ } &  \multirow{2}{*}{100000} & baseline     & 47209($\pm$ 36305) 
  & 97.14($\pm$3.34) \\
 &  & ET from $x_1\times x_2$  & 5183($\pm$ 710) & Grokked  \\
  \cmidrule{2-5}
   &  \multirow{2}{*}{50000} & baseline     & 49965($\pm$ 29621) & 97.90($\pm$ 2.67) \\
    &  &  ET from $x_1 \times x_2$ & 6535($\pm$ 3708) & Grokked  \\
\toprule
\end{tabular} }
\label{table:embedding}
\end{table}

\begin{figure}[H]
\centering
   \includegraphics[width=0.7\textwidth]{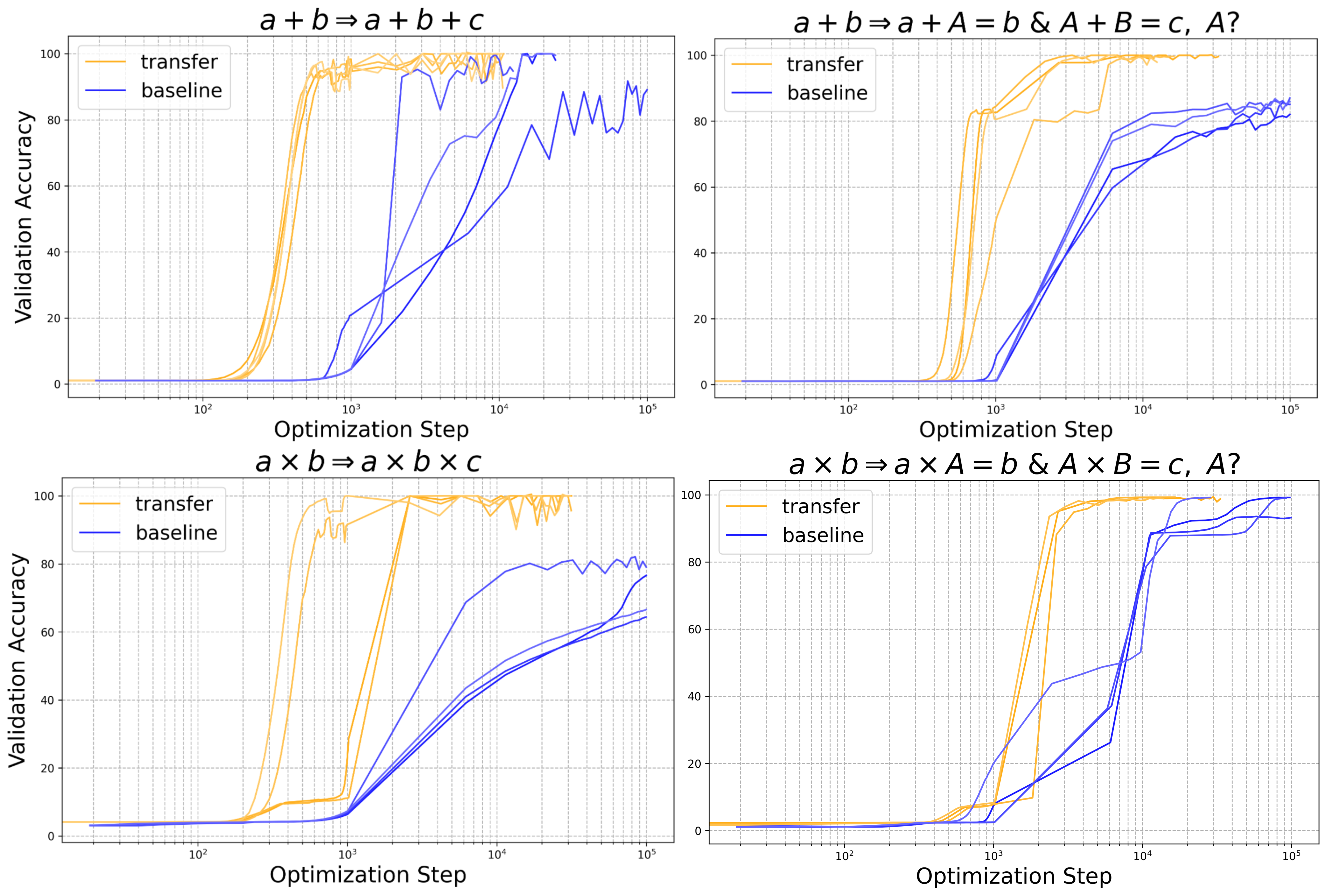} 
    \caption{The plots show that embedding transfer successfully accelerates grokking for all tasks considered: the composition of operations (left) and solving for unknowns in a system of equations (right).}
    \label{fig:grok:embedding}
\end{figure}

\subsection{Training with a limited number of tokens}
We empirically verify that the transformer can capture algebraic properties when learning the composition of arithmetic operations. To demonstrate this, the embedding layer is transferred from a single binary operation. Although the model uses only a limited number of tokens representing $0$ to $79$ (80 in total), an acceleration of grokking is observed, as shown in Table~\ref{table:limited_token}. 

\begin{table}[h]
\caption{Number of steps required for grokking when a fewer number of tokens $(80<97)$ are used with $N=10,000$ training samples and trained via embedding transfer (ET).}
\vspace{0.3cm}
\centering
\scalebox{0.6}{\begin{tabular}{@{}cccccc@{}}
\toprule
Operation & $N$ & Method &  $\#$ of tokens &Grokking step & Final accuracy \\
\midrule
 \multirow{2}{*}{$x_1 + x_2 + x_3$} 
  &  \multirow{2}{*}{10000} & baseline & 80 & Non-grokked & 26.42($\pm$ 0.10) \\
  &              &  ET      & 80 &       61034($\pm$ 25437)    & Grokked \\ 
\midrule
 \multirow{2}{*}{$x_1 \times x_2 \times x_3$} 
  &  \multirow{2}{*}{10000} & baseline & 80& Non-grokked & 28.68($\pm$ 0.11) \\
  &                         &  ET      & 80 &      10679($\pm$ 2837)     & Grokked \\
\toprule
\end{tabular} }
\label{table:limited_token}
\end{table}

\section{Conclusion}
We explore the various approaches to accelerate the occurrence of grokking phenomena that are commonly observed in learning binary operations with transformers. Our methods feature the implementation of algebraic properties of binary operations and group structures. For general commutative binary operation, it is demonstrated that simple data augmentation contributes to the improvement of training efficiency. Exploiting the Kolmogorov-Arnold representation theorem, we justify that several different types of weight transfers accelerate the grokking. Leveraging the idea of weight transfer, we observe the acceleration of grokking in some extended tasks including the composition of arithmetic operations and a system with unknowns.

\section*{Appendix}\label{sec:app}
Some missing proofs and lemmas are provided. 
\subsection{Proof of Theorem~\ref{thm:abel}}
\label{proof_thm_abel}
Since each $C_{q_j}$ is abelian, by Lemma \ref{lem:one_dimensional}, we have an irreducible representation for $C_{q_j}$ $\rho_j: C_{q_j} \to GL_1(\mathbb{C}) \cong \mathbb{C}^\ast$. By Lemma \ref{lem:faithful}, $\rho_j$ is injective, hence $\rho_j^{-1}$ is well-defined. Define $T:\mathbb{C}\to\mathbb{C}$ as $T(a+bi)=a+\tilde{b}$ where $a,b \in \mathbb{R}, \; b = \tilde{b} \Mod{2 \pi} $.  Let $\phi_j(x):=\log(\rho_j(x))$ and $\psi_j(z):= \rho_j^{-1}(\exp(T(z)))$. We define $\rho_{<G,\circ>}:G \to (\mathbb{C}^*)^m$, $\phi_{<G,\circ>}:G \to (\mathbb{C}^*)^m$ and $\psi_{<G,\circ>}: (\mathbb{C}^*)^m \to G$ as 
\begin{align*}
    &\rho_{<G,\circ>}(x):=[\rho_1(x),\dots,\rho_m(x)], \\
    &\phi_{<G,\circ>}(x):=[\phi_1(x),\dots,\phi_m(x)] = \log(\rho_{<G,\circ>}(x)), \\
    &\psi_{<G,\circ>} := \psi_1 \times \dots \times  \psi_m \in C_{q_1} \times \dots C_{q_m} \cong G.
\end{align*}
where $\log$ applies component-wise. 
Then we have the following multiplicative representation:
\begin{align*} 
    &\psi_{<G,\circ>}( \sum_{i=1}^n \phi_{<G,\circ>}(x_i)) = \psi_{<G,\circ>}( \sum_{i=1}^n \bigoplus_{j=1}^m \log(\rho_j(x_i))) \\
    &= \psi_{<G,\circ>}\big( \bigoplus_{j=1}^m \log(    \rho_j(x_1)\rho_j(x_2) \dots \rho_j(x_n) ) \big) \\
    &= \psi_{<G,\circ>}\big( \bigoplus_{j=1}^m \log( \rho_j(x_1 \circ x_2 \circ \dots \circ x_n) ) \big) \\
    &= x_1 \circ \dots \circ x_n.
\end{align*}
where $\bigoplus$ denotes the concatenation and the second equality follows from Lemma \ref{lem:complex_log_sum}.
Since $\mathbb{C}^* \subset \mathbb{C} \cong \mathbb{R}^2$, we can regard as $\phi_{<G,\circ>}:G\to\mathbb{R}^{2m},\psi:\mathbb{R}^{2m}\to G$.

\subsection{Representation theory}

We provide necessary lemmas as well as details of the proofs.

\begin{lem}[Schur's Lemma \cite{serre1977linear}]
    Let $G$ be a group and $k$ be an algebraically closed field. Let $V,W$ be vector spaces over $k$ and $\rho_V:G \to GL(V),\rho_W:G\to GL(W)$ be irreducible representations of $G$ over the field $k$. Let $f:V\to W$ be a homomorphism from $\rho_V$ to $\rho_W$. Suppose $\rho_V=\rho_W$, then $f$ is a scalar multiplication map.
\end{lem}

\begin{lem}
\label{lem:one_dimensional}
    Let $G$ be an abelian group and $V$ be a vector space over  an algebraically closed field $k$.
    Suppose that $\rho: G\to GL(V)$ be an irreducible representation. Then $\dim_k(V)=1$.
\end{lem}

\begin{lem}[\cite{gaschutz1954endliche}]
\label{lem:faithful}
    Let $G$ be a finite abelian group. G has a faithful (injective) irreducible representation if and only if it is cyclic.
\end{lem}

\begin{lem}\label{lem:complex_log_sum}
    Let $z_1, z_2 \in \mathbb{C}^\ast$. Define $T:\mathbb{C}\to\mathbb{C}$ as $T(a+bi)=a+\tilde{b}$ where $a,b \in \mathbb{R}, \; b = \tilde{b} \Mod{2 \pi} $. Then 
    \begin{align*}
        T(\log(z_1) + \log(z_2)) = T(\log(z_1z_2)).
    \end{align*}
\end{lem}
\begin{proof}
    Let $z_1=r_1e^{i \theta_1}$, and $z_2=r_2e^{i \theta_2}$ for $r_1,r_2 \in (0,\infty), \; \theta_1,\theta_2 \in [0, 2 \pi )$. Then $\log(z_1)= \log(r_1)+i \theta_1$, $\log(z_2)=\log(r_2)+i \theta_2$. Let $\tilde{\theta} = \theta_1 + \theta_2 \Mod{2 \pi }$. We have 
    \begin{align*}
        &T(\log(z_1) + \log(z_2)) = T(\log(r_1)+\log(r_2) + i( \theta_1 +\theta_2)) \\ 
        &= \log(r_1r_2) + i \tilde{\theta}  
        = T ( \log(r_1r_2) + i (\theta_1+\theta_2) ) = T(\log(z_1z_2)).
    \end{align*}
\end{proof}


 \bibliographystyle{elsarticle-num} 
 \bibliography{ref}

\begin{thebibliography}{10}
\expandafter\ifx\csname url\endcsname\relax
  \def\url#1{\texttt{#1}}\fi
\expandafter\ifx\csname urlprefix\endcsname\relax\def\urlprefix{URL }\fi
\expandafter\ifx\csname href\endcsname\relax
  \def\href#1#2{#2} \def\path#1{#1}\fi

\bibitem{power2022grokking}
A.~Power, Y.~Burda, H.~Edwards, I.~Babuschkin, V.~Misra, Grokking: Generalization beyond overfitting on small algorithmic datasets, arXiv preprint arXiv:2201.02177 (2022).

\bibitem{trask2018neural}
A.~Trask, F.~Hill, S.~E. Reed, J.~Rae, C.~Dyer, P.~Blunsom, Neural arithmetic logic units, Advances in {N}eural {I}nformation {P}rocessing {S}ystems 31 (2018).

\bibitem{hoshen2016visual}
Y.~Hoshen, S.~Peleg, Visual learning of arithmetic operation, in: Proceedings of the AAAI Conference on Artificial Intelligence, Vol.~30, 2016.

\bibitem{liu2022omnigrok}
Z.~Liu, E.~J. Michaud, M.~Tegmark, Omnigrok: Grokking beyond algorithmic data, in: The Eleventh International Conference on Learning Representations, 2022.

\bibitem{barak2022hidden}
B.~Barak, B.~Edelman, S.~Goel, S.~Kakade, E.~Malach, C.~Zhang, Hidden progress in deep learning: Sgd learns parities near the computational limit, Advances in Neural Information Processing Systems 35 (2022) 21750--21764.

\bibitem{chughtai2023toy}
B.~Chughtai, L.~Chan, N.~Nanda, A toy model of universality: Reverse engineering how networks learn group operations (2023).
\newblock \href {http://arxiv.org/abs/2302.03025} {\path{arXiv:2302.03025}}.

\bibitem{charton2024learning}
F.~Charton, Learning the greatest common divisor: explaining transformer predictions (2024).
\newblock \href {http://arxiv.org/abs/2308.15594} {\path{arXiv:2308.15594}}.

\bibitem{fort2019goldilocks}
S.~Fort, A.~Scherlis, The goldilocks zone: Towards better understanding of neural network loss landscapes, in: Proceedings of the {AAAI} {C}onference on {A}rtificial {I}ntelligence, Vol.~33, 2019, pp. 3574--3581.

\bibitem{thilak2022slingshot}
V.~Thilak, E.~Littwin, S.~Zhai, O.~Saremi, R.~Paiss, J.~Susskind, The slingshot mechanism: An empirical study of adaptive optimizers and the grokking phenomenon, arXiv preprint arXiv:2206.04817 (2022).

\bibitem{liu2022towards}
Z.~Liu, O.~Kitouni, N.~S. Nolte, E.~Michaud, M.~Tegmark, M.~Williams, Towards understanding grokking: An effective theory of representation learning, Advances in Neural Information Processing Systems 35 (2022) 34651--34663.

\bibitem{nanda2022progress}
N.~Nanda, L.~Chan, T.~Lieberum, J.~Smith, J.~Steinhardt, Progress measures for grokking via mechanistic interpretability, in: The Eleventh International Conference on Learning Representations, 2022.

\bibitem{zhong2024clock}
Z.~Zhong, Z.~Liu, M.~Tegmark, J.~Andreas, The clock and the pizza: Two stories in mechanistic explanation of neural networks, Advances in Neural Information Processing Systems 36 (2024).

\bibitem{lyu2023dichotomy}
K.~Lyu, J.~Jin, Z.~Li, S.~S. Du, J.~D. Lee, W.~Hu, Dichotomy of early and late phase implicit biases can provably induce grokking, in: The Twelfth International Conference on Learning Representations, 2023.

\bibitem{tan2023understanding}
Z.~Tan, W.~Huang, Understanding grokking through a robustness viewpoint, arXiv preprint arXiv:2311.06597 (2023).

\bibitem{furuta2024interpreting}
H.~Furuta, M.~Gouki, Y.~Iwasawa, Y.~Matsuo, Interpreting grokked transformers in complex modular arithmetic, arXiv preprint arXiv:2402.16726 (2024).

\bibitem{gallian2021contemporary}
J.~Gallian, Contemporary abstract algebra, Chapman and Hall/CRC, 2021.

\bibitem{braun2009constructive}
J.~Braun, M.~Griebel, On a constructive proof of kolmogorov’s superposition theorem, Constructive approximation 30 (2009) 653--675.

\bibitem{zaheer2017deep}
M.~Zaheer, S.~Kottur, S.~Ravanbakhsh, B.~Poczos, R.~R. Salakhutdinov, A.~J. Smola, Deep sets, Advances in {N}eural {I}nformation {P}rocessing {S}ystems 30 (2017).

\bibitem{kingma2014adam}
D.~P. Kingma, J.~Ba, Adam: A method for stochastic optimization, arXiv preprint arXiv:1412.6980 (2014).

\bibitem{loshchilov2017decoupled}
I.~Loshchilov, F.~Hutter, Decoupled weight decay regularization, arXiv preprint arXiv:1711.05101 (2017).

\bibitem{serre1977linear}
J.-P. Serre, et~al., Linear representations of finite groups, Vol.~42, Springer, 1977.

\bibitem{gaschutz1954endliche}
W.~Gasch{\"u}tz, Endliche gruppen mit treuen absolut-irreduziblen darstellungen, Mathematische Nachrichten 12~(3-4) (1954) 253--255.

\end{thebibliography}

\end{document}